\documentclass[10pt,onecolumn,twoside]{IEEEtran}
\usepackage{amsopn}
\usepackage{amsthm}
\usepackage{etex}
\usepackage{spconf,graphicx}
\usepackage{endnotes}
\usepackage{epsfig,psfrag}
\usepackage{pst-all}
\usepackage{amssymb,amsfonts,upref,cite,epsf,color,bm}
\usepackage{color}
\usepackage{url}
\usepackage{amsmath}
\usepackage{graphicx}
\usepackage{subfig}
\usepackage{pstricks}
\usepackage{calc}
\usepackage{booktabs}
\usepackage{tikz}
\usepackage{pgfplots}

\newcommand\defeq{:=}

\DeclareMathOperator*{\argmin}{\mbox{arg\;min}}
\DeclareMathOperator*{\argmaximum}{\mbox{arg\;max}}

\newcommand\vect[1]{\mathbf #1}

\newcommand{\vx}{\vect{x}}  
\newcommand{\vy}{\vect{y}}  
\newcommand{\vz}{\vect{z}}

\newcommand{\mD}{\mathbf{D}}

\newcommand{\mW}{\mathbf{W}}

\newcommand{\signalsize}{N}
\newcommand{\maxiter}{K}

\newcommand{\graphsigs}{\mathbb{R}^{|\nodes|}} 
\newcommand{\edgesigs}{\mathbb{R}^{|\edges|}}

\newcommand{\dataset}{\mathcal{D}}
\newcommand{\edges}{\mathcal{E}}
\newcommand{\edgeset}{\mathcal{S}}
\newcommand{\nodes}{\mathcal{V}}

\newcommand{\graph}{\mathcal{G}}
\newcommand{\samplingset}{\mathcal{M}}

\newcommand{\xsig}{\vx}

\newcommand{\mU}{\mathbf{U}}
\newcommand{\mQ}{\mathbf{Q}}

\newcommand{\primslp}{\hat{\vx}_{\rm SLP}}
\newcommand{\dualslp}{\hat{\vy}_{\rm SLP}}
\usepackage{algorithm}
\usepackage{algpseudocode}
\floatname{algorithm}{Algorithm}
\algnewcommand\algorithmicinput{\textbf{Input:}}
\algnewcommand\INPUT{\item[\algorithmicinput]}
\algnewcommand\algorithmicoutput{\textbf{Output:}}
\algnewcommand\OUTPUT{\item[\algorithmicoutput]}
\newtheorem{theorem}{Theorem}

\title{On the Complexity of Sparse Label Propagation}
\name{\hspace*{-10mm} Alexander Jung 
}

\address{\normalsize Department of Computer Science, Aalto University, Finland; firstname.lastname(at)aalto.fi \\[-0.5mm]
}

\begin{document}

\onecolumn
\maketitle

\begin{abstract}
This paper investigates the computational complexity of sparse label propagation which has been proposed recently 
for processing network-structured data. Sparse label propagation amounts to a convex optimization problem and 
might be considered as an extension of basis pursuit from sparse vectors to clustered graph signals representing the 
label information contained in network-structured datasets. Using a standard first-order oracle model, we characterize 
the number of iterations for sparse label propagation to achieve a prescribed accuracy. In particular, we derive an upper 
bound on the number of iterations required to achieve a certain accuracy and show that this upper bound is sharp for 
datasets having a chain structure (e.g., time series). 
\end{abstract}


\section{Introduction}
\label{sec_intro}
A powerful approach to processing massive datasets is via using graph models. In particular, we 
consider datasets which can be characterized by an ``empirical graph'' (cf.\ \cite[Ch. 11]{SemiSupervisedBook}) 
whose nodes represent individual data points and whose edges connect data points which are similar in an 
application-specific sense. The empirical graph for a particular dataset might be obtained by (domain) expert 
knowledge, an intrinsic network structure (e.g., for social network data) or in a data-driven fashion by imposing 
smoothness constrains on observed realizations of graph signals (which serve as training data) 
\cite{ZhuLuo2017,pmlr-v51-kalofolias16,CSGraphSelJournal,JungGaphLassoSPL,JuHeck2014,QuangJung2017,QuangJung2018}. 
Besides the graph structure, datasets carry additional information in the form of labels (e.g., class membership) 
associated with individual data points. We will represent such label information as graph signals defined over 
the empirical graph \cite{Zhu02learningfrom}. 

Using graph signals for representing datasets is appealing for several reasons. Indeed, having a graph model for 
datasets facilitates scalable distributed data processing in the form of message passing over the empirical graph \cite{SandrMoura2014}. 
Moreover, graph models allow to cope with heterogeneous datasets containing mixtures of different data types, 
since they only require an abstract notion of similarity between individual data points. In particular, the structure 
encoded in the graph model of a dataset enables to capitalize, by exploiting the similarity between data points, 
on massive amounts of unlabeled data via semi-supervised learning \cite{SemiSupervisedBook}. This is important, 
since labelling of data points is often expensive and therefore label information is typically available only for a small 
fraction of the overall dataset. The labels of individual data points induce a graph signal which is defined over the 
associated empirical graph. We typically only have access to the signal values (labels) of few data points and the 
goal is learn or recover the remaining graph signal values (labels) for all other data points. 

The processing of graph signals relies on particular models for graph signals. A prominent line of uses applies 
spectral graph theory to extend the notion of band-limited signals from the time domain (which corresponds to 
the special case of a chain graph) to arbitrary graphs \cite{Zhu02learningfrom,SemiSupervisedBook,ChenMoura2014,ChenSanKov2015,ChenVarma,Gadde2014}.
These band-limited graph signals are smooth in the sense of having a small variation over well-connected subsets 
of nodes, where the variation is measured by the Laplacian quadratic form. However, our approach targets datasets whose 
labels induce piece-wise constant graph signals, i.e., the signal values (labels) of data points belonging to well connected 
subset of data points (clusters), are nearly identical. This signal model is useful, e.g., in change-point detection, image 
segmentation or anomaly detection where signal values might change abruptly \cite{SharpnackJMLR2012,TrendGraph,NetworkLasso,ChenClustered2016,FanGuan2017}. 

The closest to our work is \cite{SharpnackJMLR2012,TrendGraph,NetworkLasso} for general graph models, 
as well as a line of work on total variation-based image processing \cite{chambolle2004algorithm,pock_chambolle_2016,PrecPockChambolle2011}. 
In contrast to \cite{chambolle2004algorithm,pock_chambolle_2016,PrecPockChambolle2011}, 
which consider only regular grid graphs, our approach applies to arbitrary graph topology. 
The methods presented in \cite{SharpnackJMLR2012,TrendGraph,FanGuan2017} apply also to arbitrary graph 
topologies but require (noisy) labels available for all data points, while we consider labels available only on a small subset of nodes. 

\subsection{Contributions and Outline} 

In Section \ref{sec_setup}, we formulate the problem of recovering clustered graph signals 
as a convex optimization problem. We solve this optimization problem by applying a preconditioned variant of the 
primal-dual method of Pock and Chambolle \cite{PrecPockChambolle2011}. As detailed in Section \ref{sec_spl_Alg}, the resulting 
algorithm can be implemented as a highly scalable message passing protocol, which we coin sparse label propagation (SLP). 
In Section \ref{sec_main_results}, we present our main result which is an upper bound on the number of SLP iterations 
ensuring a particular accuracy. We also discuss the tightness of the upper bound for datasets whose empirical graph 
is a chain graph (e..g, time series). 

\subsection{Notation}
Given a vector $\vx\!=\!(x_{1},\ldots,x_{n})^{T}\in \mathbb{R}^{n}$, we define the norms $\| \vx \|_{1} \defeq \sum_{l=1}^{n} |x_{l}|$ and 
$\| \vx \|_{2} \defeq \sqrt{ \sum_{l=1}^{n} (x_{l})^{2}}$, respectively. 
The spectral norm of a matrix $\mathbf{D}$ is denoted $\| \mD \|_{2} \defeq \sup_{\| \vx \|_{2}=1} \|\mD \vx \|_{2}$. 
For a positive semidefinite (psd) matrix $\mQ \in \mathbb{R}^{n \times n}$, with spectral decomposition $\mQ = \mU{\rm diag} \{ q_{i} \}_{i=1}^{n} \mU^{T}$, 
we define its square root as $\mQ^{1/2} \defeq \mU {\rm diag} \{ \sqrt{q_{i}} \}_{i=1}^{n} \mU^{T}$. 
For a positive definite matrix $\mathbf{Q}$, we define the norm $\| \vx \|_{\mathbf{Q}} \defeq \sqrt{ \vx^{T} \mQ \vx}$. 
The signum ${\rm sign } \{ \vx \}$ of a vector $\vx=\big(x_{1},\ldots,x_{d}\big)$ is defined as the vector $\big({\rm sign } (x_{1}),\ldots,{\rm sign } (x_{d}) \big) \in \mathbb{R}^{d}$ 
with the scalar signum function 
\begin{equation} 
{\rm sign } (x_{i}) = \begin{cases} -1 & \mbox{ for } x_{i} < 0 \\ 0 & \mbox{ for } x_{i}= 0 \\ 1 & \mbox{ for } x_{i} > 0. \end{cases}
\end{equation} 

Throughout this paper we consider convex functions $g(\vx)$ whose epigraphs 
${\rm epi } \,g \defeq \{ (\vx,t):  \vx \in \mathbb{R}^{n}, g(\vx) \leq t \} \subseteq \mathbb{R}^{n} \times \mathbb{R}$ are non-empty closed convex sets \cite{RockafellarBook}. 
Given such a convex function $g(\vx)$, we denote its subdifferential at $\vx_{0} \in \mathbb{R}^{n}$ by 
\begin{equation} 
\partial g(\vx_{0}) \defeq \{ \vy \in \mathbb{R}^{n}: g(\vx) \geq g(\vx_{0})\!+\!\vy^{T}(\vx\!-\!\vx_{0}) \mbox{ for any } \vx \} \subseteq \mathbb{R}^{n} \nonumber 
\end{equation}  
and its convex conjugate function by \cite{BoydConvexBook}
\begin{equation}
\label{equ_def_convex_conjugate}
g^{*}(\hat{\vy}) \defeq \sup_{\vy \in \mathbb{R}^{n}} \vy^{T}\hat{\vy}- g( \vy). 
\end{equation} 
We can re-obtain a convex function $g(\vy)$ from its convex conjugate via \cite{BoydConvexBook} 
\begin{equation}
\label{equ_inv_conjugate}
 g(\hat{\vy}) \defeq \sup_{\vy \in \mathbb{R}^{n}} \vy^{T}\hat{\vy}- g^{*}( \vy). 
\end{equation}


\section{Problem Setting}
\label{sec_setup}

We consider network-structured datasets which are represented by an undirected weighted graph $\graph = (\nodes, \edges, \mathbf{W})$, 
referred to as the ``empirical graph'' (see Figure \ref{fig_graph_signals}). The nodes $i \in \nodes$ of the empirical graph represent individual 
data points, such as user profiles in a social network or documents of a repository. An undirected edge $\{i,j\} \in \edges$ of the empirical 
graph encodes a notion of (physical or statistical) proximity of neighbouring data points, such as profiles of befriended social network users 
or documents which have been co-authored by the same person. This network structure is identified with conditional independence relations 
within probabilistic graphical models (PGM) \cite{CSGraphSelJournal,JungGaphLassoSPL,JuHeck2014,QuangJung2017,QuangJung2018}. 

As opposed to PGM, we consider a fully deterministic graph-based model which does not invoke an underlying probability distribution for 
the observed data. In particular, given an edge $\{i,j\} \in \edges$, the nonzero value $W_{i,j}\!>\!0$ represents the amount of similarity 
between the data points $i,j \in \nodes$. The edge set $\edges$ can be read off from the non-zero pattern of the weight 
matrix $\mathbf{W} \!\in\! \mathbb{R}^{\signalsize \times \signalsize}$ since 
\begin{equation}
\label{equ_edge_set_support_weights}
\{ i , j \} \in \edges \mbox{ if and only if } W_{i,j}  > 0.  
\end{equation} 
According to \eqref{equ_edge_set_support_weights}, we could in principle handle network-structured datasets using traditional 
multivariate (vector/matrix based) methods. However, putting the emphasis on the empirical graph leads naturally to scalable 
algorithms which are implemented as message passing methods (see Algorithm \ref{sparse_label_propagation_mp} below). 

The neighbourhood $\mathcal{N}(i)$ and weighted degree (strength) $d_{i}$ of node $i \in \nodes$ are  
\begin{equation} 
\label{equ_def_neighborhood}
\mathcal{N}(i) \defeq \{ j \in \nodes : \{i,j\} \!\in\!\edges \} \mbox{, and } d_{i} \defeq \sum_{j \in \mathcal{N}(i)} W_{i,j}\mbox{, respectively.} 
\end{equation} 
In what follows we assume the empirical graph to be connected, i.e., $d_{i} > 0$ for all nodes $i \in \nodes$ and having no 
self-loops such that $W_{i,i}\!=\!0$ for all $i\!\in\!\nodes$. 
The maximum (weighted) node degree is
\begin{equation}
\label{equ_def_max_node_degree}
 d_{\rm max} \defeq \max_{i \in \mathcal{V}} d_{i} \stackrel{\eqref{equ_def_neighborhood}}{=} \max_{i \in \nodes} \sum_{j \in \mathcal{N}(i)} W_{i,j} . 
\end{equation}
 
It will be convenient to orient the undirected empirical graph $\graph=(\nodes,\edges,\mathbf{W})$, which yields the directed version 
$\overrightarrow{\graph}=(\nodes,\overrightarrow{\edges},\mW)$. The orientation amounts to declaring for each edge $e=\{i,j\}$ one 
node as the head (origin node) and the other node as the tail (destination node) denoted $e^{+}$ and  $e^{-}$, respectively. Given a 
set of edges $\edgeset \subseteq \edges$ in the undirected graph $\graph$, we denote the corresponding set of directed edges in 
$\overrightarrow{\graph}$ as $\overrightarrow{\edgeset}$. The \emph{incidence matrix} $\mD \!\in\! \mathbb{R}^{|\edges| \times |\nodes|}$ 
of $\overrightarrow{\graph}$ is \cite{SharpnackJMLR2012}
\begin{equation}
D_{e,i} = \begin{cases}  W_{e} & \mbox{ if } i = e^{+}  \\ 
				    - W_{e} & \mbox{ if } i = e^{-}  \\ 
				    0 &  \mbox{ else.}  \label{equ_def_incidence_mtx}
				    \end{cases}
\end{equation} 
If we number the nodes and orient the edges in the chain graph in Fig.\ \ref{fig_graph_signals}-(a) from left to right, 
its weighted incidence matrix would be 
\begin{equation}
\mathbf{D} = \begin{pmatrix} W_{1,2} & -W_{1,2} & 0  \\  0 & W_{2,3} & -W_{2,3} \end{pmatrix}. \nonumber
\end{equation} 
The directed neighbourhoods of a node $i\in \nodes$ are defined as $\mathcal{N}_{+}(i) \defeq \{ j \in \nodes: e=\{i,j\} \in \edges \mbox{, and } e^{+}=i \}$ and 
$\mathcal{N}_{-}(i) \defeq \{ j \in \nodes: e=\{i,j\} \in \edges \mbox{, and } e^{-}=i \}$, respectively. 
We highlight that the particular choice of orientation for the empirical graph $\graph$ has no effect on our results andmethods and will be only used for notational convenience.

\begin{figure}
\begin{center}
\includegraphics[width=.6\columnwidth,angle=0]{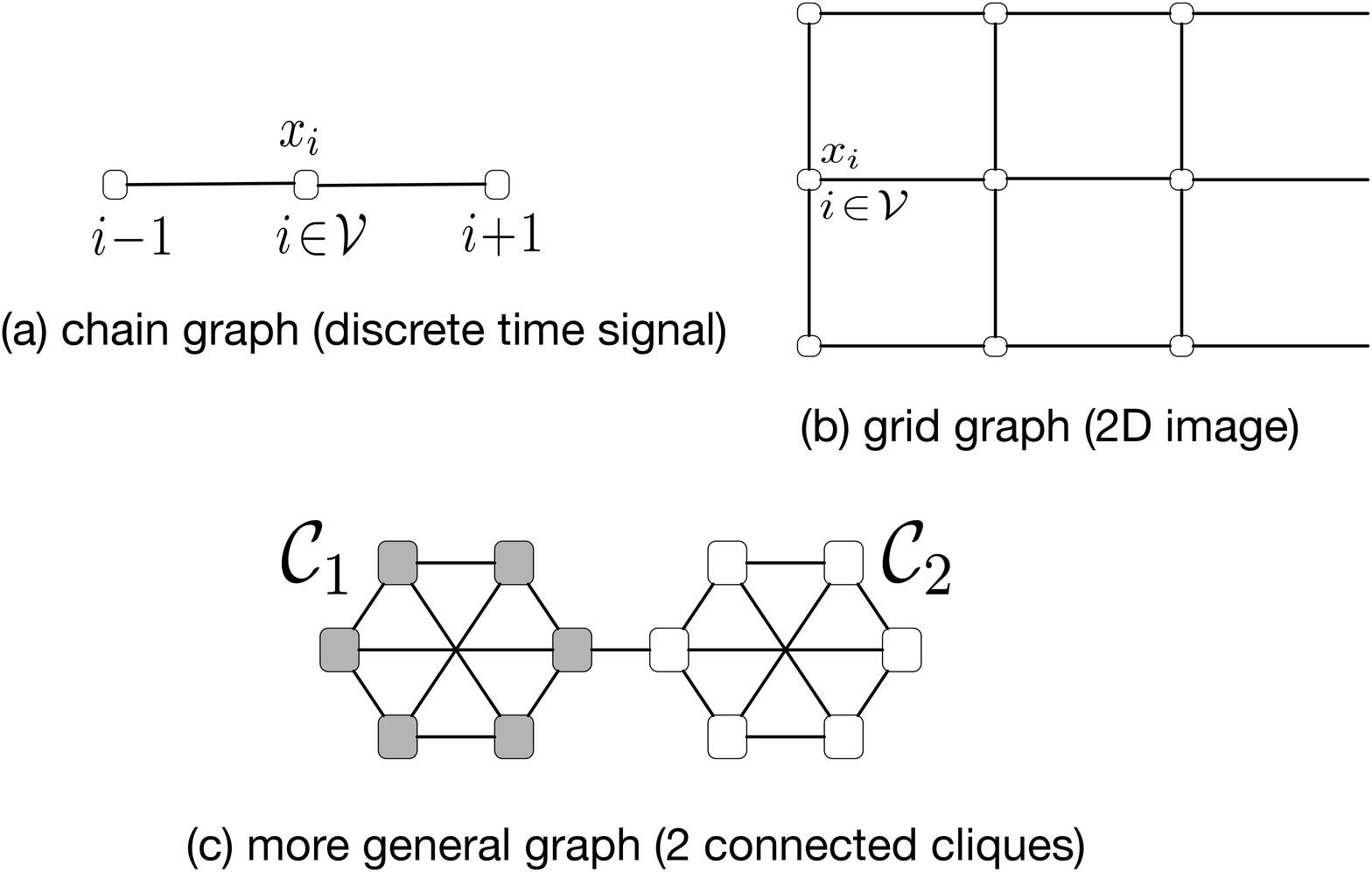}
\end{center}
\caption{\label{fig_graph_signals} Some examples of network-structured datasets with associated empirical graph being (a) a chain graph (discrete time signals), 
(b) grid graph (2D-images) and (c) a clustered graph (social networks).}
\end{figure}

In many applications we can associate each data point $i \in \nodes$ with a label $x_{i}$, e.g., in a social network application the label $x_{i}$ might encode the group membership 
of the member $i \in \nodes$ in a social network $\graph$. We interpret the labels $x_{i}$ as values of a graph signal $\xsig$ defined over the empirical graph $\graph$. 
Formally, a graph signal $\xsig \in \mathbb{R}^{\nodes}$ defined over the graph $\graph$ maps each node $i \in \nodes$ 
to the graph signal value $x[i] \in \mathbb{R}$. Since acquiring labels is often costly and error-prone, we typically have access to a few 
noisy labels $\tilde{x}_{i}$ for the data points $i \in \samplingset \subseteq \nodes$ within a (small) subset $\samplingset \subseteq \nodes$ of nodes in the empirical graph. 
Thus, we are interested in recovering the entire graph signal $\xsig$ from knowledge of its values $x[i] = \tilde{x}_{i}$ on 
a small subset $\samplingset \subseteq \nodes$ of labeled nodes $i \in \samplingset$. The signal recovery 
will be based on a clustering assumption \cite{SemiSupervisedBook}.

\noindent
{\bf Clustering Assumption (informal).}
{\it Consider a graph signal $\vx \in \mathbb{R}^{|\nodes|}$ whose signal values are the (mostly unknown) labels $\tilde{x}_{i}$ of the 
data points $z_{i} \in \dataset$. The signal values $x[i]$, $x[j]$ at nodes $i,j \in \nodes$ within a well-connected subset (cluster) of 
nodes in the empirical graph are similar, i.e., $x[i] \approx x[j]$.} 
This assumption of clustered graph signals $\vx$ can be made precise by requiring a small total variation (TV)
\begin{equation} 
\label{equ_def_TV}
\| \vx \|_{\rm TV} \defeq \sum_{\{i,j\} \in \edges} W_{i,j}  | x[j]\!-\!x[i]|.
\end{equation} 
The incidence matrix $\mD$ (cf.\ \eqref{equ_def_incidence_mtx}) allows to represent the TV of a graph signal conveniently as   
\begin{equation}
\label{equ_repr_ell_1_TV}
\| \vx \|_{\rm TV}  = \| \mathbf{D} \vx \|_{1}. 
\end{equation} 
We note that related but different measures for the total variation of a graph signal have been proposed previously (see, e.g., \cite{Chen2015,MouraDSPgraphs}). 
The definition \eqref{equ_def_TV} is appealing for several reasons. First, it conforms with the class of piece-wise constant or clustered graph 
signals which has proven useful in several applications including meteorology and binary classification \cite{LogisticNLasso,ChenVarma2016}. 
Second, as we demonstrate in what follows, the definition \eqref{equ_def_TV} allows to derive semi-supervised learning methods which can be implemented by efficient massing 
passing over the underlying empirical graph and thus ensure scalability of the resulting algorithm to large-scale (big) data.

A sensible strategy for recovering a graph signal with small TV is via minimizing the TV $\| \tilde{\vx} \|_{\rm TV}$ 
while requiring consistency with the observed noisy labels $\{ \tilde{x}_{i} \}_{i \in \samplingset}$, i.e., 
\begin{align} 
\primslp & \in \argmin_{\vx \in \graphsigs} \sum_{\{i,j\} \in \edges} W_{i,j}  | x[j]\!-\!x[i]|
\quad \mbox{s.t.} \quad x[i]\!=\!\tilde{x}_{i}  \mbox{ for all } i \in \samplingset \nonumber \\
& \stackrel{\eqref{equ_repr_ell_1_TV}}{=} 
\argmin_{\vx \in \graphsigs} \| \mD \vx \|_{1} \quad \mbox{s.t.} \quad  x[i]\!=\!\tilde{x}_{i} \mbox{ for all } i \in \samplingset.  \label{equ_min_constr}
\end{align}
The objective function of the optimization problem \eqref{equ_min_constr} is the seminorm $\|\vx\|_{\rm TV}$, which is a convex function.
\footnote{The seminorm $\|\vx\|_{\rm TV}$ is convex since it is homogeneous ($\| \alpha\vx\|_{\rm TV}\!=\!|\alpha|\|\vx\|_{\rm TV}$ 
for $\alpha \in \mathbb{R}$) and satisfies the triangle inequality ($\|\vx\!+\!\vy\|_{\rm TV} \!\leq\! \|\vx\|_{\rm TV}\!+\!\|\vy\|_{\rm TV}$).
These two properties imply convexity \cite[Section 3.1.5]{BoydConvexBook}.}  Since moreover the constraints in \eqref{equ_min_constr} 
are linear, the optimization problem \eqref{equ_min_constr} is a convex optimization problem \cite{BoydConvexBook}. Rather trivially, 
the problem \eqref{equ_min_constr} is equivalent to 
\begin{align} 
\primslp & = \argmin_{\vx \in \mathcal{Q}} \| \mD \vx \|_{1}.  \label{equ_min_constr_equiv_Q}
\end{align}
Here, we used the constraint set $\mathcal{Q} = \{ \vx : x[i] = \tilde{x}_{i} \mbox{ for all } i \in \samplingset \}$ which collects all 
graph signals $\vx \in \graphsigs$ which match the observed labels $\tilde{x}_{i}$ for the nodes of the sampling set $\samplingset$. 

The usefulness of the learning problem \eqref{equ_min_constr} depends on two aspects: (i) the deviation of the solutions of \eqref{equ_min_constr} from 
the true underlying graph signal and (ii) the difficulty (complexity) of computing the solutions of \eqref{equ_min_constr}. The first aspect 
has been addressed in \cite{NNSPFrontiers} which presents precise conditions on the sampling set $\samplingset$ and topology of 
the empirical graph $\graph$ such that any solution of \eqref{equ_min_constr} is close to the true underlying graph signal if it is (approximately) piece-wise constant 
over well-connected subsets of nodes (clusters). The focus of this paper is the second aspect, i.e., the difficulty or complexity 
of computing approximate solutions of \eqref{equ_min_constr}.


In what follows we will apply an efficient primal-dual method to solving the convex optimization problem \eqref{equ_min_constr}. 
This primal-dual method is appealing since it provides a theoretical convergence guarantee and also allows for an efficient implementation 
as message passing over the underlying empirical graph (cf.\ Algorithm \ref{sparse_label_propagation_mp} below). 
We coin the resulting semi-supervised learning algorithm sparse label propagation (SLP) since it bears some conceptual 
similarity to the ordinary label propagation (LP) algorithm for semi-supervised learning over graph models. 
In particular, LP algorithms can be interpreted as message passing methods for solving a particular recovery (or, learning) problem \cite[Chap 11.3.4.]{SemiSupervisedBook}: 
\begin{equation}
\hat{\vx}_{\rm LP}  \in \argmin_{\vx \in \graphsigs} \sum_{\{i,j\} \in \edges} W_{i,j} (x[i]-x[j])^2  \quad \mbox{s.t.} \quad  x[i] \!=\!\tilde{x}_{i} \mbox{ for all } i \in \samplingset.  \label{equ_LP_problem}
\end{equation} 
The recovery problem \eqref{equ_LP_problem} amounts to minimizing the weighted sum of squares, 
while SLP \eqref{equ_min_constr} minimize a weighted sum of absolute values, of the signal 
differences $(x[i]-x[j])^2$ arising over the edges $\{i,j\} \in \edges$ in the empirical graph $\graph$. 
It turns out that using the absolute values of signal differences instead of their squares allows SLP 
methods to accurately learn graph signals $\xsig$ which vary abruptly over few edges, e.g., clustered graph signals considered in \cite{LogisticNLasso,ChenVarma2016}. 
In contrast, LP methods tends to smooth out such abrupt signal variations. 

The SLP problem \eqref{equ_min_constr} is also closely related to the recently proposed network Lasso \cite{NetworkLasso,WhenIsNLASSO}
\begin{equation}
\hat{\vx}_{\rm nLasso} \in  \argmin_{\vx \in \graphsigs} \sum_{i \in \samplingset}  (x[i]-\tilde{x}_{i})^2  + \lambda \| \vx \|_{\rm TV} \label{equ_nLasso}.
\end{equation} 
Indeed, according to Lagrangian duality \cite{BertsekasNonLinProgr,BoydConvexBook}, by choosing $\lambda$ in \eqref{equ_nLasso} suitably, 
the solutions of \eqref{equ_nLasso} coincide with those of \eqref{equ_min_constr}. The tuning parameter $\lambda$ trades small empirical label 
fitting error $\sum_{i \in \samplingset}  (x[i]\!-\!\tilde{x}_{i})^2$ against small total variation $\|\hat{\vx}_{\rm nLasso}\|_{\rm TV}$ of the learned 
graph signal $\hat{\vx}_{\rm nLasso}$. Choosing a large value of $\lambda$ enforces small total variation of the learned graph signal, while 
using a small value for $\lambda$ puts more emphasis on the empirical error.  In contrast to network Lasso \eqref{equ_nLasso}, which requires 
to choose the parameter $\lambda$ (e.g., using (cross-)validation \cite{AGentleIntroML,hastie01statisticallearning}), the SLP method \eqref{equ_min_constr} does not require any 
parameter tuning. 


\section{Sparse Label Propagation}
\label{sec_spl_Alg}

The recovery problem \eqref{equ_min_constr} is a convex optimization problem with a non-differentiable 
objective function, which precludes the use of standard gradient methods such as (accelerated) gradient descent. 
However, both the objective function and the constraint set of the optimization problem \eqref{equ_min_constr}  
have rather a simple structure individually. This suggests the use of efficient proximal methods \cite{ProximalMethods} for solving \eqref{equ_min_constr}. In particular, 
we apply a preconditioned variant of the primal-dual method introduced by \cite{pock_chambolle}
to solve \eqref{equ_min_constr}.  

In order to apply the primal-dual method of \cite{pock_chambolle}, we reformulate \eqref{equ_min_constr}
as an unconstrained problem (see \eqref{equ_min_constr_equiv_Q}) 
\begin{align}
\label{equ_min_constr_unconstr}
\primslp & \!\in\! \argmin_{\vx \in \mathbb{R}^{|\nodes|}} f(\vx) \defeq g(\mD \vx) + h(\vx) \mbox{, with } g(\vy) \defeq \| \vy\|_{1} \mbox{ and } 
h(\vx) \defeq  \begin{cases} \infty  \mbox{ if } \vx \notin \mathcal{Q} \\ 0 \mbox{ if } \vx \in \mathcal{Q}.\end{cases}
\end{align} 
The function $h(\vx)$ in \eqref{equ_min_constr_unconstr} is the indicator function (cf. \cite{RockafellarBook}) of the convex set $\mathcal{Q}$ and can 
be described also via its epigraph 
\begin{equation}
{\rm epi } \,h = \{ (\vx,t):  \vx \in \mathcal{Q}, t \geq 0 \} \subseteq \graphsigs \times \mathbb{R}. \nonumber
\end{equation}
It will be useful to define another optimisation problem which might be considered as a dual problem to \eqref{equ_min_constr_unconstr}, i.e., 
\begin{align} 
\label{equ_dual_SLP}
\dualslp & \!\in\! \argmaximum_{\vy \in \mathbb{R}^{|\edges|}} \tilde{f}(\vy) \defeq -h^{*}(-\mD^{T} \vy) - g^{*}(\vy).
\end{align} 
Note that the objective function $\tilde{f}(\vy)$ of the dual SLP problem \eqref{equ_dual_SLP} involves the convex conjugates $h^{*}(\vx)$ and 
$g^{*}(\vy)$ (cf. \eqref{equ_def_convex_conjugate}) of the convex functions $h(\vx)$ and $g(\vy)$ which define the primal SLP problem \eqref{equ_min_constr_unconstr}.  

By elementary convex analysis \cite{RockafellarBook}, the solutions $\primslp$ of \eqref{equ_min_constr_unconstr} are characterized by the zero-subgradient condition 
\begin{equation}
\label{equ_zero_subgradient}
\mathbf{0} \in \partial f(\primslp). 
\end{equation} 
A particular class of iterative methods for solving \eqref{equ_min_constr_unconstr}, referred to as proximal methods, is obtained via 
fixed-point iterations of some operator $\mathcal{P}: \mathbb{R}^{|\nodes|} \rightarrow \mathbb{R}^{|\nodes|}$ whose fixed-points 
are precisely the solutions $\hat{\vx}_{\rm SLP}$ of \eqref{equ_zero_subgradient}, i.e., 
\begin{equation}
\label{equ_zero_subgradient_equ}
\mathbf{0} \in \partial f(\hat{\vx}_{\rm SLP}) \mbox{ if and only if } \hat{\vx}_{\rm SLP} = \mathcal{P} \hat{\vx}_{\rm SLP}. 
\end{equation} 
In general, the operator $\mathcal{P}$ is not unique, i.e., there are different choices for $\mathcal{P}$ such that 
\eqref{equ_zero_subgradient_equ} is valid. These different choices for the operator $\mathcal{P}$ in \eqref{equ_zero_subgradient_equ} 
result in different proximal methods \cite{ProximalMethods}. 

One approach to constructing the operator $\mathcal{P}$ in \eqref{equ_zero_subgradient_equ} is based on convex duality 
\cite[Thm.\ 31.3]{RockafellarBook}, according to which a graph signal $\hat{\vx}_{\rm SLP} \in \graphsigs$ solves \eqref{equ_min_constr_unconstr} 
if and only if there exists a (dual) vector $\hat{\vy} \in \mathbb{R}^{|\edges|}$ such that 
\begin{equation}
\label{equ_two_coupled_conditions}
-(\mD^{T} \dualslp) \in \partial h(\primslp) \mbox{ , and } \mD \primslp \in \partial g^{*}(\dualslp) . 
\end{equation} 
The dual vector $\dualslp \in \mathbb{R}^{|\edges|}$ represents a signal defined over the edges $\edges$ in 
the empirical graph $\graph$, with the entry $\hat{y}_{\rm SLP}[e]$ being the signal value associated with the particular edge $e \in \edges$. 

Let us now rewrite the two coupled conditions in \eqref{equ_two_coupled_conditions} as 
\begin{equation}
\label{equ_manipulated_coupled_conditions}
\primslp - {\bm \Gamma} \mD^{T} \dualslp \in \primslp+ {\bm \Gamma} \partial h(\primslp) 
\mbox{ , and } 2 {\bm \Lambda}  \mD \primslp +\dualslp \in {\bm \Lambda} \partial g^{*}(\dualslp)+ {\bm \Lambda} \mD\primslp+\dualslp , 
\end{equation}
with the invertible diagonal matrices (cf.\ \eqref{equ_edge_set_support_weights} and \eqref{equ_def_neighborhood})
\begin{equation} 
\label{equ_def_scaling_matrices}
{\bf \Lambda} \defeq (1/2) {\rm diag} \{ \lambda_{\{i,j\}} = 1/W_{i,j} \}_{\{i,j\} \in \edges} \in \mathbb{R}^{|\edges| \times |\edges|} \mbox{ and } 
{\bf \Gamma} \defeq (1/2)  {\rm diag} \{ \gamma_{i} = 1/d_{i} \}_{i \in \nodes} \in \mathbb{R}^{|\nodes| \times |\nodes|}.
\end{equation}
The specific choice \eqref{equ_def_scaling_matrices} for the matrices ${\bf \Gamma}$ and ${\bf \Lambda}$ 
can be shown to satisfy \cite[Lemma 2]{PrecPockChambolle2011}
\begin{equation}
\label{equ_convergence_condition}
\| {\bf \Gamma}^{1/2} \mD^{T} {\bf \Lambda}^{1/2} \|_{2} \leq 1/2,
\end{equation}
which will turn out to be crucial for ensuring the convergence of the iterative algorithm we will propose for solving \eqref{equ_min_constr_unconstr}. 

It will be convenient to define the resolvent operator for the functions $g^{*}(\vy)$ and $h(\vx)$ (cf. \eqref{equ_min_constr_unconstr} and \eqref{equ_def_convex_conjugate}), 
\cite[Sec. 1.1.]{PrecPockChambolle2011}
\begin{align}
(\mathbf{I} + {\bf \Lambda} \partial g^{*})^{-1} (\vy) & \defeq \argmin\limits_{\vz \in \edgesigs} g^{*}(\vz) +  (1/2) (\vy\!-\!\vz)^{T} {\bm \Lambda}^{-1}(\vy\!-\!\vz)\mbox{, and }  \nonumber \\ 
(\mathbf{I} + {\bm \Gamma} \partial h)^{-1} (\vx) & \defeq \argmin\limits_{\vz \in \graphsigs} h(\vz) + (1/2) (\vx\!-\!\vz)^{T} {\bm \Gamma}^{-1}(\vx\!-\!\vz). \label{equ_iterations_number_112}
\end{align}
We can now rewrite the optimality condition \eqref{equ_manipulated_coupled_conditions} (for $\primslp$, $\dualslp$ to be primal and dual optimal) 
more compactly as  
\begin{align}
\label{equ_condition_fix_point}
\primslp &= (\mathbf{I} + {\bm \Gamma} \partial h)^{-1} (\primslp- {\bm \Gamma} \mD^{T} \dualslp)  \\ \nonumber
\dualslp - 2(\mathbf{I} + {\bf \Lambda}  \partial g^{*})^{-1}   {\bf \Lambda}  \mD \primslp& = (\mathbf{I} +{\bf \Lambda}  \partial g^{*})^{-1}(\dualslp-  {\bf \Lambda}\mD\primslp).
\end{align}  
The characterization \eqref{equ_condition_fix_point} of the solution $\primslp \in \graphsigs$ for the SLP problem 
\eqref{equ_min_constr} leads naturally to the following fixed-point iterations for finding $\primslp$ (cf.\ \cite{PrecPockChambolle2011}) 
\begin{align}
\label{equ_fixed_point_iterations}  
\hat{\vy}^{(k+1)} &\defeq (\mathbf{I} + {\bf \Lambda}  \partial g^{*})^{-1}(\hat{\vy}^{(k)} +  {\bf \Lambda}  \mD(2\hat{\vx}^{(k)}- \hat{\vx}^{(k-1)}))\nonumber \\  
\hat{\vx}^{(k+1)} &\defeq (\mathbf{I} + {\bm \Gamma} \partial h)^{-1} (\hat{\vx}^{(k)} - {\bm \Gamma} \mD^{T} \hat{\vy}^{(k+1)}).
\end{align}  
The fixed-point iterations \eqref{equ_fixed_point_iterations} are similar to those considered in \cite[Sec. 6.2.]{pock_chambolle} for grid graphs 
arising in image processing. In contrast, the iterations \eqref{equ_fixed_point_iterations} are formulated for an arbitrary graph (network) structure 
which is represented by the incidence matrix $\mD \in \mathbb{R}^{|\edges| \times |\nodes|}$. By evaluating the application of the resolvent 
operators (cf. \eqref{equ_iterations_number_112}), we obtain simple closed-form expressions (cf. \cite[Sec. 6.2.]{pock_chambolle})
for the updates in \eqref{equ_fixed_point_iterations} yielding, in turn, Algorithm \ref{alg_sparse_label_propagation_centralized}.


\begin{algorithm}[h]
\caption{Sparse Label Propagation}{}
\begin{algorithmic}[1]
\renewcommand{\algorithmicrequire}{\textbf{Input:}}
\renewcommand{\algorithmicensure}{\textbf{Output:}}
\Require  directed empirical graph $\overrightarrow{\graph}$ with incidence matrix $\mD\!\in\! \mathbb{R}^{\overrightarrow{\edges} \times \nodes}$ 
(cf.\ \eqref{equ_def_incidence_mtx}), sampling set $\samplingset$, initial labels $\{ \tilde{x}_{i} \}_{i \in \samplingset}$. 
\Statex\hspace{-6mm}{\bf Initialize:} $k\!\defeq\!0$, $\bar{\vx} = \hat{\vx}^{(-1)}=\hat{\vx}^{(0)}=\hat{\vy}^{(0)}\!\defeq\! \mathbf{0}$, 
$\gamma_{i} \defeq 1/(2d_{i})$, $\lambda_{\{i,j\}} = 1/(2W_{i,j})$. 
\Repeat
\vspace*{2mm}
\State  $\vx  \defeq 2 \hat{\vx}^{(k)} - \hat{\vx}^{(k-1)}$
\vspace*{2mm}
\State $\hat{\vy}^{(k+1)}  \defeq \hat{\vy}^{(k)} + {\bf \Lambda}  \mD  \vx$ with ${\bf \Lambda}={\rm diag} \{ \lambda_{\{i,j\}} \}_{\{i,j\} \in \edges}$
\vspace*{2mm}
\State $\hat{y}^{(k+1)}[e]  \defeq \hat{y}^{(k+1)}[e] / \max\{1, |\hat{y}^{(k+1)}[e]| \}$  for all edges  $e \in \overrightarrow{\edges}$
\vspace*{2mm}
\State $\hat{\vx}^{(k+1)}  \defeq \hat{\vx}^{(k)} - {\bm \Gamma} \mD^{T} \hat{\vy}^{(k+1)}$ with ${\bm \Gamma}={\rm diag} \{ \gamma_{i} \}_{i \in \nodes}$
\vspace*{2mm}
\State $\hat{x}^{(k+1)}[i] \defeq \tilde{x}_{i}$  for all sampled nodes $i \in \samplingset$
\vspace*{2mm}
\State $k \defeq k+1$ 
\vspace*{2mm}
\State $\bar{\vx}^{(k)} \defeq (1-1/k) \bar{\vx}^{(k-1)}  + (1/k) \hat{\vx}^{(k)} $
\vspace*{2mm}
\Until{stopping criterion is satisfied}
\vspace*{2mm}
\Ensure labels $\hat{x}_{\rm SLP}[i] \defeq \bar{x}^{(k)}[i]$ \mbox{ for all }$i \in \nodes$
\end{algorithmic}
\label{alg_sparse_label_propagation_centralized}
\end{algorithm}
Note that the Algorithm \ref{alg_sparse_label_propagation_centralized} does not directly output the 
iterate $\hat{\vx}^{(k)}$ but its running average $\bar{\vx}^{(k)}$. Computing the running average (see step 8 in Algorithm \ref{alg_sparse_label_propagation_centralized}) 
requires only little effort but allows for a simpler convergence analysis (see the proof of Theorem \ref{thm_conv_result} in the Appendix).

One of the appealing properties of Algorithm \ref{alg_sparse_label_propagation_centralized} is that it allows for a highly scalable implementation via 
message passing over the underlying empirical graph $\graph$. 
This message passing implementation, summarized in Algorithm \ref{sparse_label_propagation_mp}, is obtained by implementing the application of the graph 
incidence matrix $\mD$ and its transpose $\mD^{T}$ (cf. steps $2$ and $5$ of Algorithm \ref{alg_sparse_label_propagation_centralized}) by local updates of 
the labels $\hat{x}[i]$, i.e., updates which involve only the neighbourhoods $\mathcal{N}(i)$, $\mathcal{N}(j)$ of all edges $\{i,j\} \in \edges$ in the empirical graph $\graph$.  

Note that executing Algorithm \ref{sparse_label_propagation_mp} does not require to collect global knowledge about the entire empircal graph 
(such as the maximum node degree $d_{\rm max}$ \eqref{equ_def_max_node_degree}) at some central processing unit. Indeed, if we associate each node in the data 
graph with a computational unit, the execution of Algorithm \ref{sparse_label_propagation_mp} requires each node $i \in \nodes$ only to store the 
values $\{ \hat{y}[\{i,j\}], W_{i,j} \}_{j \in \mathcal{N}(i)}$ and $\hat{x}^{(k)}[i]$. Moreover, the number of arithmetic operations required at each node $i \in \nodes$ 
during each time step is proportional to the number of the neighbours $\mathcal{N}(i)$. 
These characteristics allow Algorithm \ref{sparse_label_propagation_mp} to scale to massive datasets (big data) if they can be represented using 
sparse networks having a small maximum degree $d_{\rm max}$ \eqref{equ_def_max_node_degree}). The datasets generated in many important applications 
have been found to be accurately represented by such sparse networks \cite{barabasi2016network}. 

\begin{algorithm}[h]
\caption{Sparse Label Propagation as Message Passing}{}
\begin{algorithmic}[1]
\renewcommand{\algorithmicrequire}{\textbf{Input:}}
\renewcommand{\algorithmicensure}{\textbf{Output:}}
\Require directed empirical graph $\overrightarrow{\graph}=(\nodes,\overrightarrow{\edges},\mW)$, sampling set $\samplingset$, noisy labels $\{ \tilde{x}_{i} \}_{i \in \samplingset}$. 
\vspace*{3mm}
\Statex\hspace{-6mm}{\bf Initialize:} $k\!\defeq\!0$, $\bar{\vx}=\hat{\vy}^{(0)}=\hat{\vx}^{(-1)}=\hat{\vx}^{(0)}\!\defeq\!\mathbf{0}$, 
$\gamma_{i} \defeq 1/(2d_{i})$, $\lambda_{\{i,j\}} = 1/(2W_{i,j})$. 
\vspace*{3mm}
\Repeat
\vspace*{4mm}
\State for all nodes $i \in \nodes$: $x[i]  \defeq 2 \hat{x}^{(k)}[i]  - \hat{x}^{(k-1)}[i]$    
\vspace*{4mm}
\State for all edges $e=(i,j)\!\in\!\overrightarrow{\edges}$: $\hat{y}^{(k+1)}[e]  \defeq \hat{y}^{(k)}[e] +  W_{i,j} \lambda_{\{i,j\}}   (x[e^{+}] - x[e^{-}])$
\vspace*{4mm}
\State for all edges $e \in \overrightarrow{\edges}$: $\hat{y}^{(k+1)}[e]  \defeq \hat{y}^{(k+1)}[e] / \max\{1, |\hat{y}^{(k+1)}[e]| \}$  
\vspace*{4mm}
\State for all nodes $i\!\in\!\nodes$: $\hat{x}^{(k+1)}[i]\!\defeq\!\hat{x}^{(k)}[i]\!-\!\gamma_{i} \bigg[ \sum\limits_{j \in \mathcal{N}_{+}(i)} \hspace*{-3mm}W_{i,j} \hat{y}^{(k+1)}[\{i,j\}]- \hspace*{-3mm}\sum\limits_{j \in \mathcal{N}_{-}(i)} \hspace*{-3mm}W_{i,j} \hat{y}^{(k+1)}[\{i,j\}] \bigg]$    \label{algostep1} 
\vspace*{4mm}
\State for all sampled nodes $i\!\in\!\samplingset$:  $\hat{x}^{(k+1)}[i] \defeq \tilde{x}_{i}$
\vspace*{4mm}
\State $k \defeq k+1$    
\vspace*{4mm}
\State for all nodes $i\!\in\!\nodes$: $\bar{x}[i] \defeq (1-1/k)\bar{x}[i] + (1/k) \hat{x}^{(k)}[i]$
\vspace*{4mm}
\Until{stopping criterion is satisfied}
\vspace*{2mm}
\Ensure labels $\hat{x}_{\rm SLP}[i] \defeq \hat{x}^{(k)}[i]$ \mbox{ for all }$i \in \nodes$
\end{algorithmic}
\label{sparse_label_propagation_mp}
\end{algorithm}

\section{Complexity of Sparse Label Propagation}
\label{sec_main_results}
There are various options for the stopping criterion in Algorithm \ref{alg_sparse_label_propagation_centralized}, 
e.g., using a fixed number of iterations or testing for sufficient decrease of the objective function (cf.\ \cite{becker2011nesta}). 
When using a fixed number of iterations, the following characterization of the convergence rate of 
Algorithm \ref{alg_sparse_label_propagation_centralized}, we need to have a precise characterization of how many 
iterations are required to guarantee a prescribed accuracy of the resulting estimate. 
Such a characterization is provided by the following result.   

%
\begin{theorem}
\label{thm_conv_result}
Consider the sequences $\hat{\vx}^{(k)}$ and $\hat{\vy}^{(k)}$ obtained from the update rule \eqref{equ_fixed_point_iterations} and starting from some 
arbitrary initalizations $\hat{\vx}^{(0)}$ and $\hat{\vy}^{(0)}$. The averages 
\begin{equation} 
\label{equ_def_averages}
\bar{\vx}^{(\maxiter)}= (1/\maxiter) \sum_{k=1}^{\maxiter} \hat{\vx}^{(k)} \mbox{, and } \bar{\vy}^{(\maxiter)}=(1/\maxiter)\sum_{k=1}^{\maxiter} \hat{\vy}^{(k)}
\end{equation}
obtained after $\maxiter$ iterations (for $k=0,\ldots,\maxiter-1$) of \eqref{equ_fixed_point_iterations}, satisfy 
\begin{equation} 
\label{equ_bound_convergence_SLP}
\| \bar{\vx}^{(\maxiter)} \|_{\rm TV} - \| \primslp \|_{\rm TV}  \leq \frac{1}{2 \maxiter} \big( \| \hat{\vx}^{(0)} - \primslp \|^{2}_{{\bm \Gamma}^{-1}} + \| \hat{\vy}^{(0)} - \tilde{\vy}^{(\maxiter)} \|^{2}_{{\bm \Lambda}^{-1}}  \big)  
\end{equation}
with $\tilde{\vy}^{(\maxiter)} = {\rm sign } \{\mD \bar{\vx}^{(\maxiter)}  \}$. Moreover, the sequence $\| \hat{\vy}^{(0)} - \tilde{\vy}^{(\maxiter)} \|_{{\bm \Lambda}^{-1}}$, for $\maxiter=1,\ldots$, is bounded. 
\end{theorem}
\begin{proof}
see Appendix. 
\end{proof} 
According to \eqref{equ_bound_convergence_SLP}, the sub-optimality in terms of objective value function incurred by the 
output of Algorithm \ref{alg_sparse_label_propagation_centralized} after $K$ iterations is bounded as 
\begin{equation}
\label{equ_upper_bound_constant} 
\| \bar{\vx}^{(\maxiter)} \|_{\rm TV} - \| \primslp \|_{\rm TV}  \leq c/K,
\end{equation} 
where the constant $c$ does not depend on $K$ but might depend on the empirical graph via its weighted incidence matrix 
$\mD$ (cf. \eqref{equ_def_incidence_mtx}) as well as on the initial labels $\tilde{x}_{i}$. 
The bound \eqref{equ_upper_bound_constant} suggests that in order to ensure reducing the sub-optimality by a factor of two, we need 
to run Algorithm \ref{alg_sparse_label_propagation_centralized} for twice as many iterations. 

Let us now show that the bound \eqref{equ_upper_bound_constant} on the convergence speed is essentially tight. What is more, the bound 
cannot be improved substantially by any learning method, such as SLP \eqref{equ_min_constr_unconstr} or network Lasso \eqref{equ_nLasso}, 
which is implemented as message passing over the underlying empirical graph $\graph$. 
To this end we consider a dataset whose empirical graph is a weighted chain graph (see Figure \ref{fig_prove_chain})
with nodes $\nodes = \{1,\ldots,N\}$ which are connected by $N-1$ edges $\edges = \{ \{i,i+1\} \}_{i=1,\ldots,N-1}$. 
The weights of the edges are $W_{i,i+1} = 1/i$. The labels of the data points $\nodes$ induce a graph signal $\vx$ defined over 
$\graph$ with $x[i] = 1$ for all nodes $i = \{1,\ldots,N-1\}$ and $x[N]=0$. 
We observe the graph signal noise free on the sampling set $\samplingset = \{1,N\}$, resulting in the observations 
$\tilde{x}_{1} = 1$ and $\tilde{x}_{N} = 0$. According to \cite[Theorem 3]{NNSPFrontiers}, the solution 
$\primslp$ of the SLP problem \eqref{equ_min_constr_unconstr} is unique and coincides with the true underlying graph signal $\vx$. 
Thus, the optimal objective function value is $\| \primslp \|_{\rm TV} = \| \vx \|_{\rm TV} = 1/(N-1)$. On the other hand, 
the output $\bar{\vx}^{(\maxiter)}$ of Algorithm \ref{alg_sparse_label_propagation_centralized} after $\maxiter$ iterations satisfies 
$\bar{x}^{(\maxiter)}[1] = 1$ and $\bar{x}^{(\maxiter)}[i]=0$ for all nodes $i \in \{\maxiter+1,\ldots,N\}$. Thus, 
\begin{equation} 
\| \bar{\vx}^{(\maxiter)} \|_{\rm TV} \geq 1/W_{\maxiter,\maxiter+1} = 1/\maxiter,
\end{equation} 
implying, in turn, 
\begin{equation}
\label{equ_lower_bound_chain}
\| \bar{\vx}^{(\maxiter)} \|_{\rm TV}  - \| \primslp \|_{\rm TV}  \geq  1/\maxiter - 1/N. 
\end{equation} 
For the regime of $\maxiter/N \ll 1$ which is reasonable for big data applications where the number of iterations $\maxiter$ computed 
in Algorithm \ref{alg_sparse_label_propagation_centralized} is small compared to the size $N$ of the dataset, the dependency of the 
lower bound \eqref{equ_lower_bound_chain} on the number of iterations is essentially $\propto 1/\maxiter$ and therefore matches the 
upper bound \eqref{equ_upper_bound_constant}. This example indicates that, for certain structure of edge weights, chain graphs 
are among the most challenging topologies regarding the convergence speed of SLP.
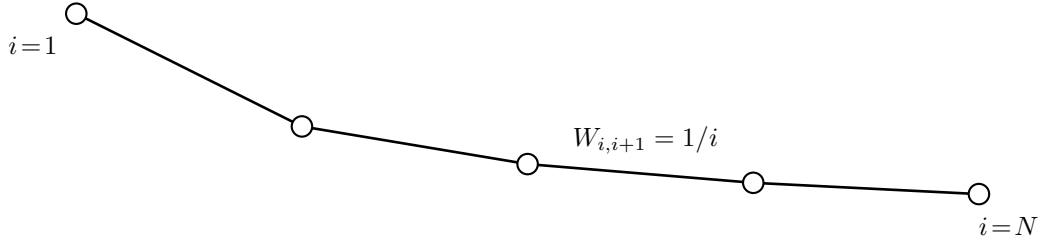
\begin{figure}
\begin{pspicture}(-1.5,-1)(5.5,2.5)
\psset{unit=3cm}
\psline[linewidth=1pt]{}(0,1)(1,0.5)
\pscircle[fillstyle=solid](0,1){0.05}
\psline[linewidth=1pt]{}(1,0.5)(2,0.333)
\pscircle[fillstyle=solid](1,0.5){0.05}
\psline[linewidth=1pt]{}(2,0.333)(3,0.25)
\pscircle[fillstyle=solid](2,0.333){0.05}
\psline[linewidth=1pt]{}(3,0.25)(4,0.2)
\pscircle[fillstyle=solid](4,0.2){0.05}
\pscircle[fillstyle=solid](3,0.25){0.05}
\rput[tl](2.2,0.5){$W_{i,i+1} = 1/i$}
\rput[tl](-0.3,0.9){$i\!=\!1$}
\rput[tl](4,0.1){$i\!=\!N$}
\end{pspicture}
\caption{\label{fig_prove_chain} The empirical graph $\graph$ is a chain graph with edge weights $W_{i,i+1} =1/i$. We aim at recovering the graph a graph signal 
from the observations $\tilde{x}_{1}= 1$ and $\tilde{x}_{N}=0$ using Algorithm \ref{sparse_label_propagation_mp}.}
\end{figure}

\section{Conclusions}
\label{sec5_conclusion}
We have studied the intrinsic complexity of sparse label propagation by deriving an upper bound 
on the number of iterations required to achieve a given accuracy. This upper bound is essentially tight 
as it cannot be improved substantially for the particular class of graph signals defined over a chain graph (such as time series). 

\appendix
\section*{Proof of Theorem \ref{thm_conv_result}}
\label{app_proof_Thm1}
Our proof closely follows the argument used for deriving \cite[Thm. 1]{pock_chambolle}. 
Let us start with rewriting the objective function of the SLP problem \eqref{equ_min_constr_unconstr} using convex conjugate functions (cf.\ \eqref{equ_inv_conjugate}) as 
\begin{equation}
\label{equ_rewrite_cost_fun_SLP}
f(\vx) =  \sup_{\vy \in \mathbb{R}^{|\edges|}} \mathcal{L}(\vx,\vy) \mbox{ with } \mathcal{L}(\vx,\vy) \defeq \vy^{T} \mD \vx+ h(\vx)- g^{*}(\vy),
\end{equation} 
so that we can reformulate the SLP problem \eqref{equ_min_constr_unconstr} equivalently as 
\begin{equation}
\label{equ_slp_primal_formulation}
\primslp \in \argmin_{\vx \in \graphsigs} \sup_{\vy \in \mathbb{R}^{|\edges|}} \mathcal{L}(\vx,\vy).
\end{equation} 
The SLP dual problem \eqref{equ_dual_SLP} is obtained by swapping the order of minimization and maximization (taking supremum): 
\begin{equation}
\label{equ_slp_dual_formulation}
\dualslp \in \argmaximum_{\vy \in \mathbb{R}^{|\edges|}} \inf_{\vx \in \graphsigs} \mathcal{L}(\vx,\vy) = \argmaximum_{\vy \in \mathbb{R}^{|\edges|}}  - h^{*}(-\mathbf{D}^{T} \vy)- g^{*}(\vy). 
\end{equation} 
According to \cite[Corollary 31.2.1]{RockafellarBook}, the optimal objective values of the primal \eqref{equ_slp_primal_formulation} and 
dual problem \eqref{equ_slp_dual_formulation} coincide, i.e., 
\begin{equation}
\label{equ_minmax_equals_maxmin}
\inf_{\vx \in \graphsigs} \sup_{\vy \in \mathbb{R}^{|\edges|}} \mathcal{L}(\vx,\vy)  =  \sup_{\vy \in \mathbb{R}^{|\edges|}}  \inf_{\vx \in \graphsigs} \mathcal{L}(\vx,\vy),
\end{equation}
and, in turn, 
\begin{equation}
\label{equ_objective_TV_minmax}
\| \primslp \|_{\rm TV} \stackrel{\primslp \in \mathcal{Q}}{=} f(\hat{\vx}_{\rm SLP}) 
\stackrel{\eqref{equ_slp_primal_formulation}}{=} \inf_{\vx \in \graphsigs} \sup_{\vy \in \mathbb{R}^{|\edges|}} \mathcal{L}(\vx,\vy)  =  \sup_{\vy \in \mathbb{R}^{|\edges|}}  \inf_{\vx \in \graphsigs} \mathcal{L}(\vx,\vy).
\end{equation}
Therefore, by combining \eqref{equ_minmax_equals_maxmin} with \cite[Lemma 36.2]{RockafellarBook}, we have that any pair $\hat{\vx}_{\rm SLP}, \hat{\vy}_{\rm SLP}$ 
consisting of a primal and dual optimal point forms a saddle point of $\mathcal{L}$, i.e., 
\begin{equation}
\label{equ_primal_dual_SLP_saddle_point}
\mathcal{L}(\primslp,\vy) \leq \mathcal{L}(\primslp,\dualslp)  \leq \mathcal{L}(\vx,\dualslp) \mbox{, for any } \vx \in \graphsigs, \vy \in  \mathbb{R}^{|\edges|},
\end{equation}
and moreover 
\begin{equation} 
\label{equ_primal_dual_SLP_saddle_point}
 \mathcal{L}(\primslp,\dualslp)  = \inf_{\vx \in \graphsigs}  \sup_{\vy \in \mathbb{R}^{|\edges|}} \mathcal{L}(\vx,\vy) =   \sup_{\vy \in \mathbb{R}^{|\edges|}}  \inf_{\vx \in \graphsigs} \mathcal{L}(\vx,\vy)   \stackrel{\eqref{equ_objective_TV_minmax}}{=} f(\primslp). 
\end{equation}

Let us analyze the effect of a single update \eqref{equ_fixed_point_iterations}. Using \eqref{equ_iterations_number_112}, 
we have for any $\vx \in \graphsigs$ and $\vy \in \mathbb{R}^{|\edges|}$, 
\begin{align}
g^{*}(\vy) & \geq g^{*}\big(\hat{\vy}^{(k+1)}\big) + \big(\hat{\vy}^{(k)} - \hat{\vy}^{(k+1)}\big)^{T}  {\bf \Lambda}^{-1} (\vy - \hat{\vy}^{(k+1)})+  \bar{\vx}^{T} \mD^{T} \big(\vy-\hat{\vy}^{(k+1)}\big)  \nonumber\\[3mm]
h(\vx) & \geq h\big(\hat{\vx}^{(k+1)}\big) +  \big(\hat{\vx}^{(k)} - \hat{\vx}^{(k+1)}\big)^{T}  {\bf \Gamma}^{-1} (\vx - \hat{\vx}^{(k+1)}) - \big( \hat{\vy}^{(k)} \big)^{T} \mD (\vx - \hat{\vx}^{(k+1)}) \label{equ_two_inequalities}
\end{align} 
with the shorthand 
\begin{equation}
\label{equ_def_bar_x} 
\bar{\vx} \defeq \mD (2\hat{\vx}^{(k)}- \hat{\vx}^{(k-1)}).
\end{equation} 
Summing the inequalities in \eqref{equ_two_inequalities} and collecting terms, 
\begin{align}
\label{equ_bound_proof_111}
\big\| \vy\!-\!\hat{\vy}^{(k)} \big\|^{2}_{{\bf \Lambda}^{-1}} + \big\| \vx\!-\!\hat{\vx}^{(k)}\big\|^{2}_{{\bf \Gamma}^{-1}} &  \geq   2 \mathcal{L}\big( \hat{\vx}^{(k+1)}, \vy \big)\!-\!2 \mathcal{L} \big( \vx, \hat{\vy}^{(k+1)} \big)   \nonumber \\[3mm]
& \hspace*{-20mm} +  \big\|\vy\!-\!\hat{\vy}^{(k+1)} \big\|^{2}_{{\bf \Lambda}^{-1}}
\!+\!\big\|\vx\!-\!\hat{\vx}^{(k+1)}\big\|^{2}_{{\bf \Gamma}^{-1}}\!+\!\big\|\hat{\vy}^{(k)}\!-\!\hat{\vy}^{(k+1)} \big\|^{2}_{{\bf \Lambda}^{-1}}\!+\!\big\|\hat{\vx}^{(k)}\!-\!\hat{\vx}^{(k+1)}\big\|^{2}_{{\bf \Gamma}^{-1}} \nonumber \\[3mm]
& \hspace*{-20mm}+ 2(\hat{\vx}^{(k+1)}\!-\!\bar{\vx})^{T} \mD^{T}\big( \hat{\vy}^{(k+1)}\!-\!\vy \big). 
\end{align} 
We can develop the final summand on the right hand side of \eqref{equ_bound_proof_111} as 
\begin{align}
\label{equ_bound_112}
(\hat{\vx}^{(k+1)}\!-\!\bar{\vx})^{T} \mD^{T}\big( \hat{\vy}^{(k+1)}\!-\!\vy \big) & \stackrel{\eqref{equ_def_bar_x}}{=} (\hat{\vx}^{(k+1)}\!-\!(2 \hat{\vx}^{(k)} - \hat{\vx}^{(k-1)}))^{T} \mD^{T}\big( \hat{\vy}^{(k+1)}\!-\!\vy \big)  \nonumber \\[3mm]
& = (\hat{\vx}^{(k+1)}\!-\!\hat{\vx}^{(k)})^{T} \mD^{T}\big( \hat{\vy}^{(k+1)}\!-\!\vy \big)  -  (\hat{\vx}^{(k)}\!-\!\hat{\vx}^{(k-1)})^{T} \mD^{T}\big( \hat{\vy}^{(k)}\!-\!\vy \big)  \nonumber \\[3mm]
& -  (\hat{\vx}^{(k)}\!-\!\hat{\vx}^{(k-1)})^{T} \mD^{T}\big( \hat{\vy}^{(k+1)}\!-\!\hat{\vy}^{(k)} \big). 
\end{align} 
The last term in \eqref{equ_bound_112} can be further developed, with the shorthand  $\kappa \defeq \big\|  {\bf \Gamma}^{1/2} \mathbf{D}^{T} {\bf \Lambda}^{1/2} \big\|_{2}$, as
\begin{align}
\label{equ_bound_11222}
 (\hat{\vx}^{(k)}\!-\!\hat{\vx}^{(k-1)})^{T} \mD^{T}\big( \hat{\vy}^{(k+1)}\!-\!\hat{\vy}^{(k)} \big) & \stackrel{(a)}{\leq}  \kappa \big\| \hat{\vx}^{(k)}\!-\!\hat{\vx}^{(k-1)} \big\|_{{\bf \Gamma}^{-1}}  
 \big\| \hat{\vy}^{(k+1)}\!-\!\hat{\vy}^{(k)} \big\|_{{\bf \Lambda}^{-1}}   \nonumber \\ 
 & \stackrel{(b)}{\leq}  (1/2)\kappa\big( \big\| \hat{\vx}^{(k)}\!-\!\hat{\vx}^{(k-1)} \big\|^{2}_{{\bf \Gamma}^{-1}}  + 
 \big\| \hat{\vy}^{(k+1)}\!-\!\hat{\vy}^{(k)} \big\|^{2}_{{\bf \Lambda}^{-1}} \big) 
\end{align} 
where step $(a)$ follows from the Cauchy-Schwarz inequality and step $(b)$ is due to the elementary inequality $2ab \leq a^2 + b^2$ valid for any $a,b \in \mathbb{R}$. 
Combining \eqref{equ_bound_11222} and \eqref{equ_bound_112} with \eqref{equ_bound_proof_111}, 
\begin{align} 
\label{equ_bound_1333}
\big\| \vy\!-\!\hat{\vy}^{(k)} \big\|^{2}_{{\bf \Lambda}^{-1}} + \big\| \vx\!-\!\hat{\vx}^{(k)}\big\|^{2}_{{\bf \Gamma}^{-1}} &  \geq   2 \mathcal{L}\big( \hat{\vx}^{(k+1)}, \vy \big)\!-\!2 \mathcal{L} \big( \vx, \hat{\vy}^{(k+1)} \big)   \nonumber \\[3mm]
& \hspace*{-20mm} +  \big\|\vy\!-\!\hat{\vy}^{(k+1)} \big\|^{2}_{{\bf \Lambda}^{-1}}
\!+\!\big\|\vx\!-\!\hat{\vx}^{(k+1)}\big\|^{2}_{{\bf \Gamma}^{-1}}\!+\!(1\!-\!\kappa)  \big\|\hat{\vy}^{(k)}\!-\!\hat{\vy}^{(k+1)} \big\|^{2}_{{\bf \Lambda}^{-1}}\!+\!\big\|\hat{\vx}^{(k)}\!-\!\hat{\vx}^{(k+1)}\big\|^{2}_{{\bf \Gamma}^{-1}} \nonumber \\[3mm]
& \hspace*{-20mm} -\!\kappa \big\|\hat{\vx}^{(k)}\!-\!\hat{\vx}^{(k-1)}\big\|^{2}_{{\bf \Gamma}^{-1}} + 2 (\hat{\vx}^{(k+1)}\!-\!\hat{\vx}^{(k)})^{T} \mD^{T}\big( \hat{\vy}^{(k+1)}\!-\!\vy \big)  -  2 (\hat{\vx}^{(k)}\!-\!\hat{\vx}^{(k-1)})^{T} \mD^{T}\big( \hat{\vy}^{(k)}\!-\!\vy \big). 
\end{align} 
Summing \eqref{equ_bound_1333} for $k=0,\ldots,\maxiter-1$, 
\begin{align}
\label{equ_bound_13332}
 2 &  \sum_{k=1}^{\maxiter} \big( \mathcal{L}\big( \hat{\vx}^{(k)}, \vy \big)\!-\! \mathcal{L} \big( \vx, \hat{\vy}^{(k)} \big) \big)   \nonumber \\ 
 & + \big\| \vy\!-\!\hat{\vy}^{(\maxiter)} \big\|^{2}_{{\bf \Lambda}^{-1}} + \big\| \vx\!-\!\hat{\vx}^{(\maxiter)}\big\|^{2}_{{\bf \Gamma}^{-1}} + 
 (1\!-\!\kappa)  \sum_{k=1}^{\maxiter} \big\|\hat{\vy}^{(k)}\!-\!\hat{\vy}^{(k-1)} \big\|^{2}_{{\bf \Lambda}^{-1}}  \nonumber \\[3mm]
&+  (1\!-\!\kappa)  \sum_{k=1}^{\maxiter-1} \big\|\hat{\vx}^{(k)}\!-\!\hat{\vx}^{(k-1)} \big\|^{2}_{{\bf \Gamma}^{-1}}  + \big\|\hat{\vx}^{(\maxiter)}\!-\!\hat{\vx}^{(\maxiter-1)} \big\|^{2}_{{\bf \Gamma}^{-1}}  \nonumber \\[3mm]
& \leq  \big\|\vy\!-\!\hat{\vy}^{(0)} \big\|^{2}_{{\bf \Lambda}^{-1}} +\big\|\vx\!-\!\hat{\vx}^{(0)}\big\|^{2}_{{\bf \Gamma}^{-1}} + 2 (\hat{\vx}^{(\maxiter)}\!-\!\hat{\vx}^{(\maxiter-1)})^{T} \mD^{T}\big( \hat{\vy}^{(\maxiter)}\!-\!\vy \big). 
\end{align} 
Similar to \eqref{equ_bound_11222}, we can also develop the last term on the right hand side of \eqref{equ_bound_13332} as 
\begin{align}
\label{equ_bound_1444}
 (\hat{\vx}^{(\maxiter)}\!-\!\hat{\vx}^{(\maxiter-1)})^{T} \mD^{T}\big( \hat{\vy}^{(\maxiter)}\!-\!\vy \big) \leq (1/2)\kappa\big( \big\| \hat{\vx}^{(\maxiter)}\!-\!\hat{\vx}^{(\maxiter-1)} \big\|^{2}_{{\bf \Gamma}^{-1}}  + 
 \big\| \hat{\vy}^{(\maxiter)}\!-\!\vy \big\|^{2}_{{\bf \Lambda}^{-1}} \big). 
\end{align}
Combining \eqref{equ_bound_1444} with \eqref{equ_bound_13332}, 
\begin{align}
\label{equ_bound_15555}
 2 &  \sum_{k=1}^{\maxiter} \big( \mathcal{L}\big( \hat{\vx}^{(k)}, \vy \big)\!-\! \mathcal{L} \big( \vx, \hat{\vy}^{(k)} \big) \big)   \nonumber \\ 
 & + (1-\kappa) \big\| \vy\!-\!\hat{\vy}^{(\maxiter)} \big\|^{2}_{{\bf \Lambda}^{-1}} + \big\| \vx\!-\!\hat{\vx}^{(\maxiter)}\big\|^{2}_{{\bf \Gamma}^{-1}} + (1\!-\!\kappa)  \sum_{k=1}^{\maxiter} \big\|\hat{\vy}^{(k)}\!-\!\hat{\vy}^{(k-1)} \big\|^{2}_{{\bf \Lambda}^{-1}}  \nonumber \\[3mm]
&+  (1\!-\!\kappa)  \sum_{k=1}^{\maxiter-1} \big\|\hat{\vx}^{(k)}\!-\!\hat{\vx}^{(k-1)} \big\|^{2}_{{\bf \Gamma}^{-1}}  + (1-\kappa) \big\|\hat{\vx}^{(\maxiter)}\!-\!\hat{\vx}^{(\maxiter-1)} \big\|^{2}_{{\bf \Gamma}^{-1}}  \nonumber \\[3mm]
& \leq  \big\|\vy\!-\!\hat{\vy}^{(0)} \big\|^{2}_{{\bf \Lambda}^{-1}} +\big\|\vx\!-\!\hat{\vx}^{(0)}\big\|^{2}_{{\bf \Gamma}^{-1}}, 
\end{align} 
which holds for any $\vx \in \graphsigs$ and $\vy \in \mathbb{R}^{|\edges|}$. 

Since, for fixed $\vy$, the quantity $\mathcal{L}(\vx,\vy)$ is a convex function of $\vx$ and, for a fixed $\vx$, it is a concave function of $\vy$, we have
\begin{equation}
\label{equ_convexity_L_func}
\big( \mathcal{L}\big( \bar{\vx}^{(\maxiter)}, \vy \big)\!-\! \mathcal{L} \big( \vx, \bar{\vy}^{(\maxiter)} \big) \big) \leq (1/\maxiter)  
\sum_{k=1}^{\maxiter} \big( \mathcal{L}\big( \hat{\vx}^{(k)}, \vy \big)\!-\! \mathcal{L} \big( \vx, \hat{\vy}^{(k)} \big) \big)
\end{equation} 
with the averages $\bar{\vx}^{(\maxiter)}$ and $\bar{\vy}^{(\maxiter)}$ (cf. \eqref{equ_def_averages}). 
Combining \eqref{equ_convexity_L_func} with \eqref{equ_bound_15555}, and using the particular choice 
$\vx = \primslp$ and $\vy = \tilde{\vy}^{(\maxiter)}\defeq {\rm sign} \{ \mathbf{D} \bar{\vx}^{(\maxiter)} \}$, 
\begin{equation}
\label{equ_bound_almost_there}
 \mathcal{L}\big( \bar{\vx}^{(\maxiter)}, \tilde{\vy}^{(\maxiter)} \big)\!-\! \mathcal{L} \big( \primslp, \bar{\vy}^{(\maxiter)} \big)   \leq \frac{1}{2 \maxiter} 
\big( \| \hat{\vx}^{(0)} - \primslp \|^{2}_{{\bm \Gamma}^{-1}} + \| \hat{\vy}^{(0)} - \tilde{\vy}^{(\maxiter)} \|^{2}_{{\bm \Lambda}^{-1}}  \big). 
\end{equation} 
The bound \eqref{equ_bound_convergence_SLP} follows from \eqref{equ_bound_almost_there} by noting 
\begin{align}
&  \mathcal{L}\big( \bar{\vx}^{(\maxiter)}, \tilde{\vy}^{(\maxiter)} \big)\!-\! \mathcal{L} \big( \primslp, \bar{\vy}^{(\maxiter)} \big)   \nonumber \\[3mm]
 & = \big( \mathcal{L}\big( \bar{\vx}^{(\maxiter)}, \tilde{\vy}^{(\maxiter)} \big)\!-\! \mathcal{L}\big( \primslp, \dualslp \big)\big)
 \!+\!  \underbrace{\big(\! \mathcal{L} \big( \primslp, \dualslp \big)  - \mathcal{L} \big( \primslp, \bar{\vy}^{(\maxiter)} \big)  \big)}_{\stackrel{\eqref{equ_primal_dual_SLP_saddle_point}}{\geq}0}   \nonumber \\[3mm]
 & \geq \mathcal{L}\big( \bar{\vx}^{(\maxiter)}, \tilde{\vy}^{(\maxiter)} \big)\!-\! \mathcal{L}\big( \primslp, \dualslp \big) \nonumber \\[3mm] 
 & \stackrel{\eqref{equ_primal_dual_SLP_saddle_point}}{=}    \mathcal{L}\big( \bar{\vx}^{(\maxiter)}, \tilde{\vy}^{(\maxiter)} \big)\!-\! f(\primslp) \nonumber \\[3mm] 
  & \stackrel{\eqref{equ_rewrite_cost_fun_SLP}}{=}   \big\|  \bar{\vx}^{(\maxiter)} \big\|_{\rm TV}\!-\! f(\primslp).  
\end{align}

It remains to verify the sequence $\|\tilde{\vy}^{(\maxiter)}-\hat{\vy}^{(0)} \|_{{\bf \Lambda}^{-1}}$ to be bounded. 
To this end, we evaluate \eqref{equ_bound_15555} for the particular choice $\vx=\primslp, \vy=\dualslp$, for which 
$\big( \mathcal{L}\big( \hat{\vx}^{(k)}, \hat{\vy}_{\rm SLP} \big)\!-\! \mathcal{L} \big( \hat{\vx}_{\rm SLP}, \hat{\vy}^{(k)} \big) \big)  \geq 0$ (cf. \eqref{equ_primal_dual_SLP_saddle_point}).
This yields 
\begin{align}
\label{equ_bound_1666}
 \big\| \primslp\!-\!\hat{\vx}^{(\maxiter)}\big\|^{2}_{{\bf \Gamma}^{-1}} 
& \leq  \big\|\dualslp\!-\!\hat{\vy}^{(0)} \big\|^{2}_{{\bf \Lambda}^{-1}} +\big\|\primslp\!-\!\hat{\vx}^{(0)}\big\|^{2}_{{\bf \Gamma}^{-1}}, 
\end{align} 
 which implies that $\big\| \!\hat{\vx}^{(\maxiter)}\big\|^{2}_{{\bf \Gamma}^{-1}}$ and, in turn, 
 $\|\tilde{\vy}^{(\maxiter)}-\hat{\vy}^{(0)} \|_{{\bf \Lambda}^{-1}}=\big\|{\rm sign} \big\{ \mathbf{D} \bar{\vx}^{(\maxiter)} \big\}-\hat{\vy}^{(0)} \big\|_{{\bf \Lambda}^{-1}}$ is bounded.

\vskip 0.2in
\bibliographystyle{abbrv}
\bibliography{SLPBib}

\end{document}